\documentclass{article}
\usepackage{geometry}
\usepackage[utf8]{inputenc}
\usepackage{amsmath}
\usepackage{amsthm}
\usepackage{amsfonts} 
\usepackage{bm}
\usepackage{graphicx}
\usepackage{hyperref}
\usepackage{algorithm}
\usepackage{algpseudocode}
\usepackage[numbers]{natbib}
\usepackage{booktabs}
\usepackage{xcolor}
\usepackage{algorithm}
\usepackage{algpseudocode}
\usepackage{hyperref}
\usepackage{dsfont} 
\newtheorem{theorem}{Theorem}
\newtheorem{lemma}{Lemma}
\newtheorem{proposition}{Proposition}

\newcommand{\R}{\mathbb{R}}
\DeclareMathOperator{\E}{\mathbb{E}}

\title{PurSAMERE: Reliable Adversarial Purification via Sharpness-Aware Minimization of Expected \\ Reconstruction Error}
\author{Vinh Hoang, Sebastian Krumscheid, Holger Rauhut, Raúl Tempone}

\author{
	Vinh Hoang$^{1, 2}$,
	Sebastian Krumscheid$^{3}$,
	Holger Rauhut$^{4}$,
	Raúl Tempone$^{5}$
	\\[1ex]
	{\small $^{1}$Department of Mathematics, RWTH-Aachen University, Germany} \\
	{\small $^{2}$Institute for a sustainable Hydrogen Economy (IHE),} \\ {\small Forschungszentrum Jülich, Germany} \\
	{\small $^3$Scientific Computing Center and Institute for Applied and}  \\ 
	{\small Numerical Mathematics, Karlsruhe Institute of Technology, Germany} \\
	{\small $^4$Department of Mathematics \& Munich Center for Machine Learning,} \\ 
	{\small Ludwig Maximilian University of Munich, Germany} \\
	{\small $^5$Computer, Electrical and Mathematical Sciences and Engineering,} \\
	{\small KAUST, Saudi Arabia} \\[1ex]
	{\small \texttt{vi.hoang@fz-juelich.de}, 
		\texttt{sebastian.krumscheid@kit.de},} \\
	{\small \texttt{rauhut@math.lmu.de}, \texttt{raul.tempone@kaust.edu.sa}}
}

\begin{document}
	
	\maketitle
	
	\begin{abstract}
		We propose a novel deterministic purification method to improve adversarial robustness by mapping a potentially adversarial sample toward a nearby sample that lies close to a mode of the data distribution, where classifiers are more reliable. We design the method to be deterministic to ensure reliable test accuracy and to prevent the degradation of effective robustness observed in stochastic purification approaches when the adversary has full knowledge of the system and its randomness. 
		We employ a score model trained by minimizing the expected reconstruction error of noise-corrupted data, thereby learning the structural characteristics of the input data distribution. Given a potentially adversarial input, the method searches within its local neighborhood for a purified sample that minimizes the expected reconstruction error under noise corruption and then feeds this purified sample to the classifier. 
		During purification, sharpness-aware minimization is used to guide the purified samples toward flat regions of the expected reconstruction error landscape, thereby enhancing robustness. 
		We further show that, as the noise level decreases, minimizing the expected reconstruction error biases the purified sample toward local maximizers of the Gaussian-smoothed density; under additional local assumptions on the score model, we prove recovery of a local maximizer in the small-noise limit.
		Experimental results demonstrate significant gains in adversarial robustness over state-of-the-art methods under strong deterministic white-box attacks. 
	\end{abstract}
	
	
	\section{Introduction}

	Neural networks have achieved remarkable success in various domains, including computer vision, natural language processing, and reinforcement learning.
	However, they are known to be vulnerable to adversarial attacks, where small perturbations to the input data can lead to incorrect predictions~\cite{szegedy2013intriguing, goodfellowSS14}, posing a challenge for evaluating the reliability of neural networks in real-world applications. 
	To address this issue, there are two common approaches to enhancing the adversarial robustness of neural networks: adversarial training and adversarial purification. 
	Adversarial training involves augmenting the training dataset with adversarial examples, which are generated by perturbing the original data~\cite{madry2018towards}.
	Adversarial purification aims to remove adversarial perturbations from the input data before feeding it to the classifier. 
	Adversarial purification can be classified into two main categories: manifold-based methods and density-based methods.
	Manifold-based purification methods assume that clean data lie on a low-dimensional manifold embedded in the high-dimensional input space. The purification process involves projecting the adversarial sample onto this manifold to recover the clean data. Samangouei et al.~\cite{samangouei2018defensegan} proposed Defense-GAN, which uses generative adversarial networks (GANs)~\cite{goodfellow2014generative} to learn the data manifold and purify adversarial samples by finding the closest point on the GAN manifold to the adversarial sample. Shi et al.~\cite{shi2021online} introduced an online adversarial purification method based on autoencoders, which reconstructs clean data from adversarial samples by minimizing the reconstruction error.
	Density-based purification methods aim to move adversarial samples toward regions of high data density, where the classifier is expected to be more robust. Song et al.~\cite{song2018pixeldefend} proposed using generative models with tractable likelihood~\cite{van2016pixel,salimans2017pixelcnn} to purify adversarial samples by moving them toward the local maximum of the data density.
	In~\cite{pmlr-v139-yoon21a}, Yoon et al.~used score-based generative models to purify adversarial samples using simplified Langevin dynamics with random starting points. 
	More recently, the diffusion-based purification approach, which uses diffusion models to denoise adversarial samples, was introduced in~\cite{pmlr-v162-nie22a} and further improved in~\cite{li2025adbm}.
	
	Adversarial purification may introduce gradient obfuscation, which is even more pronounced for methods that rely heavily on stochasticity. Athalye et al.~\cite{pmlr-v80-athalye18a} demonstrated that gradient obfuscation can lead to unreliable test accuracies. The aim of this paper is to develop a deterministic purification method that not only enhances adversarial robustness but also provides reliable robustness under strong white-box adversarial attacks. 
	
	The adaptive attack is an adversarial attack specifically designed for the purification-based defense to address the gradient obfuscation~\cite{pmlr-v80-athalye18a}. For the purification methods that involve stochasticity, the Expectation over Transformation (EoT) technique is combined with the adaptive attack to improve its effectiveness. An even stronger adversarial attack is the white-box deterministic adaptive attack where the adversary has full knowledge of not only the classifier and the purification model but also the random generators. Li et al.~\cite{li2025adbm} showed that under strong adaptive attacks with EoT, the adversarial accuracy of stochastic purification methods significantly decreases. Moreover, Liu et al.~\cite{liu2025towards} conducted numerical tests with the deterministic white-box attacks and showed that stochastic purification methods may lose most of their adversarial accuracy. These results suggest that a purification method can be deemed effective only if it maintains robustness under both strong adaptive attacks and deterministic white-box attacks.
	
	In this paper, we introduce a novel deterministic purification approach named Purification via Sharpness-Aware Minimization of the Expected Reconstruction Error (PurSAMERE). 
	Our method purifies an adversarial sample by finding a nearby sample that minimizes the expected reconstruction error of the noise-corrupted data. This error is evaluated using a score-based generative model. The proposed purification method favors purified samples that lie in regions where the data density is concave. We show that, under certain assumptions, the purified sample converges to a local maximum of the data density function in the asymptotic limit as the noise covariance approaches zero. The intuition behind our method is that the classifier tends to be more accurate for inputs where the density function is high, and the local density maximum is stable with respect to (w.r.t.) the perturbation of the original input. Additionally, we use sharpness-aware minimization (SAM)~\cite{foret2021sharpnessaware} during the purification process to enhance the robustness of the purified samples by driving them toward regions of flat expected reconstruction error. 
	
	Since the expected reconstruction error is empirically estimated via a Monte Carlo method, a small degree of randomness is introduced. In all experiments, we fix the Monte Carlo noise samples, making the purification process completely deterministic. Our deterministic white-box evaluation assumes that the attacker has full knowledge of this fixed randomness.
	
	In contrast to prior methods that focus solely on manifold consistency or
	solely on density maximization, PurSAMERE adopts a \emph{hybrid} approach that combines both objectives: it drives the purified sample toward local maxima of the data density function while simultaneously ensuring that the purified sample lies as close as possible to the manifold implicitly defined by the expected reconstruction error. Our method has three main advantages:
	
	\begin{itemize}
		\item First, it is deterministic by design; therefore its test-time robustness accuracy does not degrade under deterministic white-box attacks and hence is more \emph{reliable}. Moreover, EoT is not needed for evaluating the robustness of deterministic purification methods, which reduces the computational cost for robustness evaluation.  
		\item Second, by minimizing the expected reconstruction error, our method drives purified samples toward \emph{high-density regions} of a Gaussian-smoothed distribution, and—under additional local structural assumptions—admits a recovery result for a local maximizer of a model potential.
		\item Third, by employing SAM during the purification process, our method effectively \emph{mitigates the impact of approximation errors} in the learned score function, which can introduce spurious local minima in the expected reconstruction error landscape. This enhances the reliability of the purification process and contributes to improved adversarial robustness.
	\end{itemize}
	
	\section{Related works}
	
	\paragraph{Generative models via conditional expectation} Denoising autoencoders (DAEs) were originally proposed as a representation learning approach that learns to recover clean data from noisy observations, thereby encouraging robustness to perturbations~\cite{vincent2008extracting}. The DAE is indeed an approximation of the conditional expectation given noise-corrupted data. Moreover, subsequent theoretical analyses~\cite{alain2014regularized, MR2839543} showed that, for small corruption levels, DAEs learn the score of the underlying data distribution, establishing a link to score matching~\cite{JMLR:v6:hyvarinen05a, hyvarinen2007connections}. This perspective later motivated and influenced the development of score-based and diffusion generative models~\cite{NEURIPS2019_3001ef25, NEURIPS2020_92c3b916, ho2020denoising}, which have achieved state-of-the-art results in image generation tasks.
	
	\paragraph{Adversarial purification} Song et al.~\cite{song2018pixeldefend} used generative models with tractable likelihoods~\cite{van2016pixel,salimans2017pixelcnn} to purify adversarial samples by moving them toward local maxima of the data density. 
	Samangouei et al.~\cite{samangouei2018defensegan} employed generative adversarial networks (GANs)~\cite{goodfellow2014generative} to purify adversarial samples by finding the closest point on the GAN manifold to the adversarial sample. 
	The accuracy of both approaches was shown to be significantly diminished under strong adaptive attacks, as reported in~\cite{pmlr-v80-athalye18a}. 
	Our method shares similar motivations with both approaches, i.e., we aim to find a local maximum of the data density that is at the same time the sample in the input space that has minimum distance to the manifold implicitly defined by the expected reconstruction error.
	However, our method is more robust thanks to the combination of three aspects: the score-based generative models can capture the complex characteristics of the data distribution, the expectation over noise corruption drives the purified samples toward the local maxima where the density function is concave, and the use of SAM during purification helps to mitigate the effect of approximation errors in the learned score function. 
	Shi et al.~\cite{shi2021online} introduced several adversarial purification methods, including one based on minimizing the reconstruction error of autoencoders. Our method differs by minimizing the expected reconstruction error under noise corruption rather than the reconstruction error of a clean sample or a single noisy sample. This particular formulation allows us to connect the purified sample to a local maximum of a smoothed approximation of the data density function. Moreover, unlike the autoencoder-based approach of~\cite{shi2021online}, which does not generalize well to complex datasets such as CIFAR-10, our score-based approach effectively captures intricate data characteristics.
	Yoon et al.~\cite{pmlr-v139-yoon21a} employed score-based generative models for adversarial sample purification using simplified Langevin dynamics with random initialization. Nie et al.~\cite{pmlr-v162-nie22a} proposed diffusion-based purification methods that denoise adversarial examples via diffusion models. Both approaches involve substantial randomness during purification. 
	In contrast, our method is deterministic and yields not only higher but also more reliable adversarial accuracy.
	
	\paragraph{Adversarial attack} Common attacks on image classifiers are Fast Gradient Sign Method (FGSM)~\cite{goodfellowSS14}, Projected Gradient Descent (PGD)~\cite{madry2018towards}, and the Carlini \& Wagner (C\&W)~\cite{CW17} attack. These attacks construct adversarial examples by optimizing a perturbation using gradient information while enforcing constraints on the input. For models that include preprocessing steps such as purification, gradients may be obfuscated; the adaptive attack~\cite{pmlr-v80-athalye18a} was proposed to address this. EoT is commonly combined with the adaptive attack to increase effectiveness against stochastic purification procedures. A still stronger adversary is the deterministic adaptive attack, in which the attacker has full knowledge of the classifier, the purification procedure, and the random-number generator. Moreover, the randomness involved in the purification is fixed, and the purification becomes deterministic. Although this deterministic setting may be less practical, it obviates the need for EoT, reduces test-time cost, and yields a reliable lower bound on the adversarial accuracy. 
	
	\section{Method}
	
	\subsection{Learning the data distribution via score-based generative models}
	
	In this work, we focus on the image classification task. Nevertheless, the proposed purification method is generic and can be extended to other data modalities. Let $X$ denote the random variable corresponding to input images, taking values in the compact set $\mathcal{X} = [0,1]^d$, where $d$ is the dimension of the input images. We assume that $X$ admits a probability density function, denoted by $p_X$.  Let $x$ be a clean image and let $x_{\mathrm{adv}}$ denote an adversarial example of $x$ w.r.t.\ a classifier $h$. The corresponding adversarial perturbation is defined as $\delta_{\mathrm{adv}} := x_\mathrm{adv} - x$ and is assumed to satisfy $\lVert \delta_{\mathrm{adv}} \rVert_p \leq \varrho_{\mathrm{adv}}$, where $\varrho_{\mathrm{adv}}$ denotes the adversarial perturbation budget. Typical choices of $p$ are $1$, $2$, and $\infty$. The objective of purification is to recover the clean image $x$ from its adversarial version $x_\mathrm{adv}$ via a purification map $\operatorname{Pur}$ such that $\operatorname{Pur}(x_\mathrm{adv}) \approx x$.
	
	A class of adversarial purification methods is based on learning data features using a conditional expectation model, which is subsequently employed to purify adversarial examples. Let $Y_\sigma$ denote a noise-corrupted version of $X$, defined as 
	\begin{equation}
		Y_\sigma = X + \sigma \varXi,
	\end{equation}
	where $\varXi \sim \mathcal{N}(0, \mathbf{I}_d)$ is a Gaussian random vector independent of $X$. The probability density function of $Y_\sigma$ is given by the convolution of $p_X$ with the Gaussian kernel as
	\begin{equation}
		p_{Y_\sigma}(y) = \int_{\R^d} \frac{1}{{(2\pi\sigma^2)}^{d/2}} \exp\left(-\frac{\lVert y - x \rVert_2^2}{2 \sigma^2}\right) p_X(x) \, \mathrm{d}x.
	\end{equation}
	For sufficiently small values of $\sigma$, $p_{Y_\sigma}$ is a smoothed approximation of the data density $p_X$.
	
	The conditional expectation $\mathbb{E}[X \mid Y_\sigma]$ is the minimum mean square error (MMSE) estimator of $X$ given $Y_\sigma$, i.e.,
	\begin{equation}
		\label{eq:mmse}
		\mathbb{E}[X \mid Y_\sigma] = \arg\min_{g \in L_2(\R^d, \, p_{Y_\sigma})} \mathbb{E}\left[\lVert X - g(Y_\sigma) \rVert_2^2\right],
	\end{equation}
	where $L_2(\R^d, \; p_{Y_\sigma})$ denotes the space of measurable functions $g : \R^d \to \R^d$ that are square-integrable w.r.t.\ $p_{Y_\sigma}$. 
	Using Tweedie's formula, the conditional expectation can be expressed as
	\begin{equation}
		\label{eq:Tweedies_formula}
		\mathbb{E}[X \mid Y_\sigma = y] = y + \sigma^2 s(y;\sigma),
	\end{equation}
	where $s(\cdot;\sigma)$ denotes the score function, defined as 
	\begin{equation}
		s(y; \sigma) = \nabla_y \log p_{Y_\sigma}(y).
	\end{equation}
	We use a deep neural network $s_\theta$ to approximate the score function $s$. Owing to the MMSE characterization of the conditional expectation in~\eqref{eq:mmse} and Tweedie's formula~\eqref{eq:Tweedies_formula}, the score function can be learned by minimizing the denoising score matching objective~\cite{MR2839543},
	\begin{equation}
		\label{eq:score_matching}
		\theta = \arg\min_{\theta'} \, \mathbb{E}_{X, \varXi}\!\left[\lVert \varXi + \sigma\, s_{\theta'}(X + \sigma \varXi; \sigma) \rVert_2^2\right],
	\end{equation}
	where the expectation is approximated by the empirical average over the training data and the noise samples.
	Indeed, solving~\eqref{eq:score_matching} is equivalent to minimizing the expected squared error between the true score function $s$ and the score model $s_\theta$, since
	\begin{equation}
		\mathbb{E}_{X, \varXi}\!\left[\lVert \varXi + \sigma\, s_{\theta}(X + \sigma \varXi;\sigma)\rVert_2^2 \right] = \sigma^2 \mathbb{E}_{X, \varXi}\!\left[\lVert s(X + \sigma \varXi;\sigma) - s_{\theta}(X + \sigma \varXi;\sigma)\rVert_2^2\right] + \operatorname{Const},
	\end{equation}
	where the constant term is independent of $\theta$.
	The proof of this equivalence is provided in Appendix~\ref{appendix:score_matching_equivalence}. In practice, multiple values of $\sigma$ are used to learn the score function at different noise levels, which helps capture the data distribution more accurately~\cite{NEURIPS2019_3001ef25, NEURIPS2020_92c3b916}.
	
	Once the score function is learned, it can be used to purify the adversarial inputs. 
	For example, in~\cite{pmlr-v139-yoon21a}, Yoon et al.~used a simplified version of Langevin dynamics to sample from the distribution $p_{Y_\sigma}$. 
	The method employs the score function as the gradient of the log-density to iteratively update the adversarial inputs toward regions of high probability density.
	This approach is motivated by the approximation $p_X \approx p_{Y_\sigma}$ with sufficiently small values of $\sigma$. 
	However, this algorithm merely seeks points at which the score function vanishes, which may correspond to saddle points of the density $p_{Y_\sigma}$ or even spurious zeros of the learned score function $s_\theta$ arising from approximation errors in the score model. 
	
	In this paper, we propose a novel purification approach based on the local minimization of the expected reconstruction error. Our method drives purified samples toward local maxima of the density $p_{Y_\sigma}$ while mitigating the impact of approximation errors in the score function.
	
	\subsection{Purification via sharpness-aware minimization of the expected reconstruction error (PurSAMERE)}
	
	We define the expected reconstruction error for a given $x$ as 
	\begin{equation}
		\label{eq:ere}
		R(x; \sigma) = \mathbb{E}_{\varXi} \left[ \lVert \varXi + \sigma s(x+\sigma\varXi; \sigma) \rVert_2^2 \right ], \quad \varXi \sim \mathcal{N}(0, \mathbf{I}_d).
	\end{equation}
	Using Tweedie's formulation, $\mathbb{E}[X \mid Y = x + \sigma\varXi] = x + \sigma\varXi + \sigma^2 s(x + \sigma\varXi; \sigma)$, the function $R$ can be expressed as
	\begin{equation}
		\label{eq:ere_vs_ce}
		R(x; \sigma) \equiv \dfrac{1}{\sigma^2}\mathbb{E}_{\varXi} \left[ \lVert x - \mathbb{E}[X \mid Y_\sigma=x + \sigma \varXi] \rVert_2^2 \right ],
	\end{equation}
	which measures the expected squared distance between $x$ and its conditional expectation reconstruction under additive noise corruption, $\mathbb{E}[X \mid Y_\sigma = x + \sigma \varXi]$.
	
	As can be observed from~\eqref{eq:ere_vs_ce}, our formulation of the expected reconstruction error is normalized by $\sigma^2$. This normalization helps stabilize the numerical optimization used within PurSAMERE, where multiple noise scales are employed and sequentially decreased toward zero.
	
	The expectation of $R(X;\sigma)$ w.r.t.\ $X$ is
	\begin{equation}
		\mathbb{E}_{X}\!\left[R(X;\sigma)\right]
		\;=\;\frac{1}{\sigma^{2}}\,
		\mathbb{E}\!\left[\|X-\mathbb{E}[X\mid Y_{\sigma}]\|_{2}^{2}\right].
	\end{equation}
	Since $\mathbb{E}[X\mid Y_\sigma]$ is the MMSE estimator of $X$ given $Y_\sigma$, we can compare it against the trivial estimator $g(Y_\sigma)=Y_\sigma$ to obtain
	\begin{equation}
		\mathbb{E}\!\left[\|X-\mathbb{E}[X\mid Y_{\sigma}]\|_{2}^{2}\right]
		\;\le\;\mathbb{E}\!\left[\|X-Y_{\sigma}\|_{2}^{2}\right]
		\;=\;\sigma^{2}\mathbb{E}\!\left[\|\Xi\|_{2}^{2}\right]
		\;=\;\sigma^{2}d,
	\end{equation}
	and hence $\mathbb{E}_{X}[R(X;\sigma)]\le d$. 
	Because of the normalization by $\sigma^2$, $R(x;\sigma)$ has a dimension-dependent baseline. In purification, we therefore rely on its variation w.r.t.\ $x$. Empirically, within a local neighborhood, the $x$-dependent part of $R(x;\sigma)$ is typically smaller for in-distribution inputs and larger for out-of-distribution or adversarial inputs.
	
	Motivated by this observation, we propose to purify a given possibly adversarial sample $x_{\mathrm{adv}}$ by minimizing the expected reconstruction error $R(x; \sigma)$ in a small neighborhood of $x_{\mathrm{adv}}$. Furthermore, we employ sharpness-aware minimization (SAM)~\cite{foret2021sharpnessaware} to enhance the robustness of the purification by driving the purified samples toward flat regions of the expected reconstruction error, thereby avoiding spurious minima. The purification is formulated as the following constrained optimization problem:
	\begin{equation}
		\label{eq:purification_via_reconstruction}
		x^\star = \arg\min_{x \in [0, 1]^d} \;\max_{\lVert \varepsilon \rVert_q \leq \varrho_{\mathrm{sam}}}R(x + \varepsilon;\sigma), \quad \text{s.t.} \quad \lVert x - x_{\mathrm{adv}} \rVert_q \leq \varrho_{\mathrm{pur}},
	\end{equation}
	where typical values of $q$ are $1$, $2$, and $\infty$. In this work, we set $q = 2$. Here, $\varrho_{\mathrm{pur}} > 0$ is the purification radius, which controls the size of the search neighborhood $B_2(x_{\mathrm{adv}}, \varrho_{\mathrm{pur}})$, and $\varrho_{\mathrm{sam}} > 0$ is the sharpness radius, which controls the size of the SAM perturbation neighborhood $B_2(x, \varrho_{\mathrm{sam}})$. Unlike conventional SAM, which operates on model parameters during training, we apply a SAM-like worst-case objective in the input space during purification.

	This \textit{purification-time} SAM step searches for a point $x^\star$ within the neighborhood of $x_{\mathrm{adv}}$ that minimizes $R_{\theta}$ even under the worst-case input perturbation, thereby ensuring that the minimizer lies in a relatively flat basin of $R_{\theta}$. Given the score model $s_\theta$ learned via~\eqref{eq:score_matching}, we approximate $R(x;\sigma)$ by
	\begin{equation}
		\label{eq:ere_approx}
		R_\theta(x; \sigma) = \mathbb{E}_{\varXi} \left[ \lVert \varXi + \sigma s_\theta(x+\sigma\varXi; \sigma) \rVert_2^2 \right ].
	\end{equation}

	The purification criterion has two main benefits. 
	\begin{itemize}
		\item First, minimizing $R(x; \sigma)$ in a neighborhood of $x_{\mathrm{adv}}$ encourages the purifier to find a nearby point $x^\star$ with small reconstruction error under the conditional expectation model. This effectively steers $x^\star$ toward a local maximizer of the smoothed density $p_{Y_\sigma}$, assuming that such a point exists within the $\ell_2$-ball $B_2(x_{\mathrm{adv}}, \varrho_{\mathrm{pur}})$, as discussed in Section~\ref{sec:theoretical_analysis}.
		\item Second, purification based on $R_\theta$ is sensitive to approximation errors in the learned score function $s_\theta$, 
		which can introduce spurious local minima of $R_\theta(x; \sigma)$ that do not correspond to genuine density maxima. To mitigate this, we employ SAM to guide purified samples toward flat regions of $R_\theta(x; \sigma)$, which helps suppress sharp minima induced by approximation errors.
	\end{itemize}

	\subsection{Relation between purified samples and the local maxima of the data density}\label{sec:theoretical_analysis}

	In this section, we analyze the behavior of the minimizers of the expected reconstruction error $R(x; \sigma)$ defined in~\eqref{eq:ere}. In particular, we show that minimizing $R(x; \sigma)$ drives $x$ toward local maxima of the log-density $\log p_X$ as illustrated in Fig.~\ref{fig:ere_vs_logpx}.
	
	\begin{figure} [h!]
		\centering
		\includegraphics[width=0.5\textwidth]{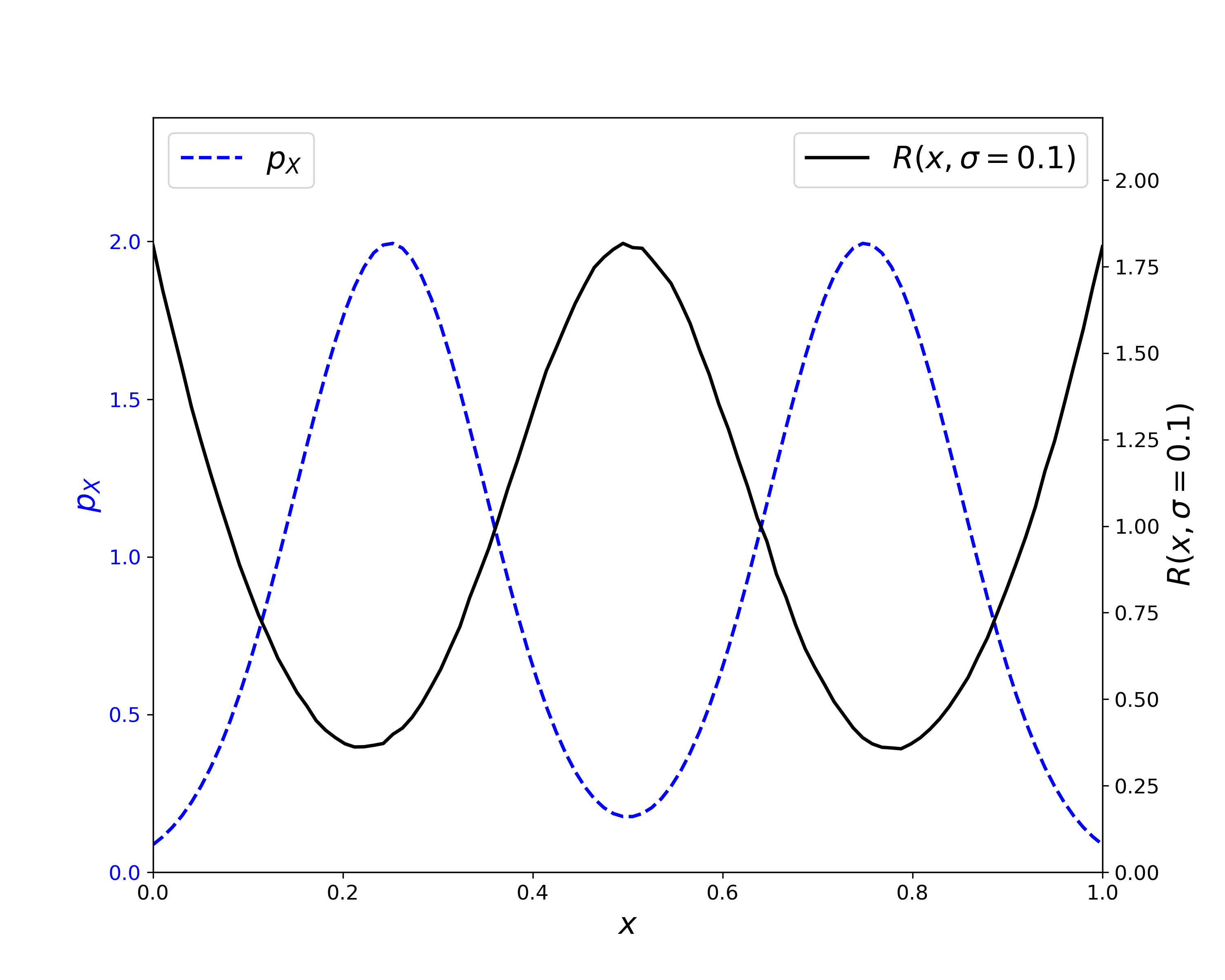} 
		\caption{The expected reconstruction error $R(x; \sigma)$ of a truncated Gaussian mixture model in 1D, $p_X(x) \propto 0.5 \; \mathcal{N}(x; 0.25, 0.1) + 0.5 \; \mathcal{N}(x; 0.75, 0.1)$. The expected reconstruction error $R(x; \sigma)$ with $\sigma=0.1$ (right) attains local minima approximately at local maxima of the density $\log p_X$ (left).}
		\label{fig:ere_vs_logpx}
	\end{figure}

	\begin{proposition}[Expected Reconstruction Error Expansion]\label{thm:ere_expansion}
		Let $s:\mathbb{R}^d\times[0,\sigma_{\max}]\to\mathbb{R}^d$ and let $\varXi\sim\mathcal{N}(0,I_d)$. Assume that for every $x\in\mathbb{R}^d$ and $\xi\in \R^d$, the map $\sigma \mapsto s(x+\sigma\xi,\sigma)$ is twice continuously differentiable on $[0,\sigma_{\max}]$.
		Define
		\[
		\phi(x,\sigma,\xi)
		:= \frac{d^2}{d\sigma^2}s(x+\sigma\xi,\sigma).
		\]
		Fix a compact set $K \subset \R^d$. Assume there exist constants $C>0$, and $m\ge 0$,
		such that for all $x\in K$, all $\sigma\in[0,\sigma_{\max}]$, and all $\xi\in\mathbb{R}^d$,
		\[
		\max_{1\le i\le d} |\phi_i(x,\sigma,\xi)|
		\;\le\; C\bigl(1+\|\xi\|_2^m\bigr).
		\]
		Then, as $\sigma\to 0$,
		\begin{equation}
			\label{eq:ere_taylor_expansion}
			R(x; \sigma) = d +\sigma^2\|s(x,0)\|_2^2 + 2\sigma^2\,\operatorname{tr}\bigl(\nabla_x s(x,0)\bigr) + \mathcal{O}(\sigma^3), \quad x \in K,
		\end{equation}
		where $\nabla_x s(x; 0)$ denotes the Jacobian of $s(x; 0)$ w.r.t.\ $x$.
	\end{proposition}

	The proof of Proposition~\ref{thm:ere_expansion} is given in Appendix~\ref{appendix:ere_expansion}. For $s(x, \sigma)$ to be twice continuously differentiable at $\sigma=0$ for all $x \in \mathbb{R}^d$, it is necessary for $p_X$ to be strictly positive on $\mathbb{R}^d$. For a compactly supported density $p_X$, the score function $s(x, 0)$ is undefined outside of its support. In this case, we consider $s(x, 0)$ as the score function of the Gaussian-smoothed density with a fixed small noise scale $\sigma_{\min} \ll 1$, so that $s(x, 0)$ is well defined for all $x \in \mathbb{R}^d$.

	Let us analyze the terms on the right-hand side of~\eqref{eq:ere_taylor_expansion} in Proposition~\ref{thm:ere_expansion}, assuming that its assumptions hold and $\sigma$ is sufficiently small. The first term, $d$, is independent of $x$; thus it does not affect the purification process. The second term $\sigma^2 \|s(x; 0)\|_2^2$ penalizes points where the score magnitude is large. Minimizing this term drives $s(x; 0) \rightarrow 0$, which corresponds to the stationary points of the log-density $\log p_{X}(x)$. Since stationary points alone may correspond to maxima, minima, or saddle points, this term by itself is insufficient to characterize typical data samples.
	
	The third term provides additional geometric discrimination. Since $s(x, 0) = \nabla_x \log p_{X}(x)$, the third term $2 \sigma^2 \operatorname{tr}(\nabla_x s(x; 0))$ equals $2 \sigma^2 \operatorname{tr}(\nabla_x^2 \log p_{X}(x))$, which is proportional to the trace of the Hessian of $\log p_{X}(x)$. Although this term cannot guarantee that the Hessian of $\log p_{X}(x^\star)$ is negative definite, minimizing it biases $x$ toward regions where $\log p_{X}(x)$ is locally concave.
	As a result, minimizing $R(x; \sigma)$ drives $x$ toward local maxima of the log-density $\log p_X(x)$, and hence toward peaks of $p_X(x)$, provided such peaks exist within the search neighborhood. Therefore, our choice of objective function $R(x; \sigma)$ for purification is more discriminative than simply forcing $s(x; \sigma) = 0$ as in~\cite{pmlr-v139-yoon21a}.

	To gain theoretical insights into the behavior of the minimizer of $R_\theta(x; \sigma)$~\eqref{eq:ere_approx}, we examine the case where $s_\theta$ is piecewise linear w.r.t.\ $x$. This assumption covers certain neural networks with ReLU activation functions and without layer normalization. Under certain assumptions, Theorem~\ref{thm:local_max_density_informal} shows that the minimizer of $R_\theta(x; \sigma)$ converges to the local maximizer of the log-density $f_\theta$, where $s_\theta = \nabla_x f_\theta$, at an exponential rate as $\sigma \to 0$. The detailed statement and proof of Theorem~\ref{thm:local_max_density_informal} are given in Appendix~\ref{appendix:local_max_density}.
	
	\begin{theorem}[Local Recovery of Density Maximizer (Informal)]~\label{thm:local_max_density_informal}
		Consider a score model $s_\theta(x,\sigma)$ that is piecewise affine,
		derives from a potential function $f_\theta$ (so $s_\theta = \nabla_x f_\theta$),
		and is locally $\mu$-strongly concave inside a region $D \in [0, 1]^d$.
		Assume the maximizer $x_\sigma^\star$ of $f_\theta(\cdot,\sigma)$ lies
		strictly inside $D$, at least with $\ell_2$-distance $\sqrt{d}\varrho$ away from its boundary.
		Let $x_\sigma'$ be any minimizer of $R_\theta(\cdot,\sigma)$ over $D_{-\sqrt{d}\,\varrho}$, where 
		\[
		D_{-\sqrt{d}\,\varrho} := \{x \in D: \inf_{y \notin D} \|x - y\|_2 \geq \sqrt{d}\,\varrho \}
		\]
		is assumed to be nonempty. If the noise level satisfies $\sigma < \rho$, then
		\[
		\|x_\sigma' - x_\sigma^\star\|_2
		\;\le\;
		\dfrac{\sqrt{C}}{\sigma \mu} \exp \left (-\dfrac{d}{8}\left(\dfrac{\varrho}{\sigma} - 1\right)^2 \right),
		\]
		where $C > 0$ is a constant independent of $\sigma$.
		In particular, $x_\sigma'$ converges to $x_\sigma^\star$ exponentially fast 
		as $\sigma \to 0$.
	\end{theorem}

	Under the assumptions of Theorem~\ref{thm:local_max_density_informal}, minimizing the expected reconstruction error $R_\theta(x; \sigma)$ w.r.t.\ $x$ recovers the local maximizer of $f_\theta(\cdot, \sigma)$ as $\sigma \rightarrow 0$. In this limit, $f_\theta(\cdot, \sigma)$ converges to the log-data density $\log p_X$, implying that the minimizer of $R_\theta(x; \sigma)$ approaches the true density maximizer. In practice, we use $\sigma \geq \sigma_\mathrm{min} > 0$, so the theorem’s conditions are idealized; nonetheless, it supports the intuition that our purifier drives inputs toward high-density regions of $p_X$. Such regions correspond to typical data and are therefore expected to exhibit increased robustness against adversarial perturbations.
	
	\section{Algorithm}

	In this paper, we use the score models developed in~\cite{NEURIPS2020_92c3b916} for our purification method. These score models are trained by minimizing the objective function~\eqref{eq:score_matching} for a sequence of noise levels $\sigma_1 > \sigma_2 > \cdots > \sigma_L$, where $\{ \sigma_i \}_{i=1}^L$ is a geometric sequence. We solve the optimization problem~\eqref{eq:purification_via_reconstruction} by sequentially refining $x$ over $L$ decreasing noise levels. Starting from $\sigma_1$ with the initialization $x_0 = x_\mathrm{adv}$, we update $x$ over $L$ iterations down to $\sigma_L$ according to
	\begin{equation}
		\begin{aligned}
			x_0 &= x_\mathrm{adv}, \\
			x^{+}_{k} & = x_{k-1} + \varrho_{\mathrm{sam}} \frac{\nabla_x \widehat{R}_\theta(x_{k-1}; \sigma_{k})}{\|\nabla_x \widehat{R}_\theta(x_{k-1}; \sigma_{k})\|_2}, \\
			x^{\text{max}}_{k} &= \operatorname{Projection}_{[0, 1]^d} (x^{+}_{k}),\\
			x'_{k} &=  x_{k-1} - \eta_k \nabla_x \widehat{R}_\theta(x^{\text{max}}_{k}; \sigma_{k}),\\
			x_{k} &= \operatorname{Projection}_{[0, 1]^d \cap B_2(x_{\mathrm{adv}}, \varrho_{\mathrm{pur}})} (x_{k}'),
		\end{aligned}
	\end{equation}
	for $k=1, \dots, L$. Here, $\eta_k$ is the step size at iteration $k$, $B_2(x_{\mathrm{adv}}, \varrho_{\mathrm{pur}})$ is the $\ell_2$-ball centered at $x_{\mathrm{adv}}$ with radius $\varrho_{\mathrm{pur}}$, and $\widehat{R}_\theta$ is an empirical approximation of the expected reconstruction error $R_\theta$ defined in~\eqref{eq:ere_approx},
	\begin{equation}\label{eq:reconstruction_error_mc}
		\widehat{R}_\theta(x; \sigma_k) = \frac{1}{m} \sum_{i=1}^m \| \xi_{ki} + \sigma_k s_\theta(x + \sigma_k\xi_{ki}; \sigma_k) \|_2^2,
	\end{equation}
	where $\{\xi_{ki}\}_{k=1,\dots,L;\, i=1,\dots,m}$ are i.i.d.~samples of $\varXi$.

	To update the purified samples, we use Adam~\cite{kingma2014adam} as the optimizer. The complete algorithm is summarized in Algorithm~\ref{alg:pumere}.
	\begin{algorithm}[H]
		\caption{PurSAMERE}
		\label{alg:pumere}
		\begin{algorithmic}[1]
			\Require Adversarial sample $x_{\mathrm{adv}}$, score model $s_\theta$ trained for noise levels including $\{\sigma_i\}_{i=1}^L$, purification radius $\varrho_\mathrm{pur}$, SAM radius $\varrho_\mathrm{sam}$, number of Monte Carlo samples $m$, Adam optimizer $\mathcal{A}$. Learning rate max and min: $\eta_{\max} = 0.1,\; \eta_{\min} = 0.001 $. Optional:
			generator $G$ and seeds $\{\operatorname{sd}_i\}_{i=1}^L$.
			
			\State $x_0 \gets x_{\mathrm{adv}}$
			\State Set Adam optimizer $\mathcal{A}$ with initial learning rate $\eta_{\max}$
			\For{$k = 1$ to $L$}
			\State Using generator $G$ with seed $\operatorname{sd}_{k}$ to generate $m$ i.i.d.~samples $\{\xi_{ki}\}_{i=1}^m$ from $\mathcal{N}(0, \mathbf{I}_d)$
			\State Compute the approximate gradient $\nabla_x \widehat{R}_\theta(x_{k-1}; \sigma_{k})$ using back-propagation based on~\eqref{eq:reconstruction_error_mc}
			\State $x_{k}^{+} \gets$ $x_{k-1} + \varrho_{\mathrm{sam}} \dfrac{\nabla_x \widehat{R}_\theta(x_{k-1}; \sigma_{k})}{\|\nabla_x \widehat{R}_\theta(x_{k-1}; \sigma_{k})\|_2}$
			\State $x_{k}^{\text{max}} \gets$ Projection $x_{k}^{+}$ onto $[0, 1]^d$
			\State Compute the approximate gradient $\nabla_x \widehat{R}_\theta(x_{k}^{\text{max}}; \sigma_{k})$ using back-propagation based on~\eqref{eq:reconstruction_error_mc}
			\State $x_{k}' \gets$ Update the state of Adam optimizer $\mathcal{A}$ with $\nabla_x \widehat{R}(x_{k}^{\text{max}} ; \sigma_{k})$
			\State $x_{k} \gets$ Projection $x_{k}'$ onto $[0, 1]^d \cap B_2(x_{\mathrm{adv}}, \varrho_{\mathrm{pur}})$
			\State $lr = \eta_{\min} + (\eta_{\max} - \eta_{\min}) \times (k-1)/(L-1)$ 
			\State Update learning rate of Adam optimizer $\mathcal{A}$ to $lr$
			\EndFor\\
			\Return Purified sample $x_L$
		\end{algorithmic}
	\end{algorithm}

	In Algorithm~\ref{alg:pumere}, the learning rate is linearly reduced from $0.1$ at the largest noise scale $\sigma_1$ to $0.001$ at the smallest noise scale $\sigma_L$. This procedure is essential to ensure stable and precise updates at fine noise scales, while allowing sufficiently large exploration steps at coarse noise scales.

	Although our purification method is deterministic by design, the Monte Carlo approximation of the expected reconstruction error~\eqref{eq:reconstruction_error_mc} introduces a small degree of randomness. This randomness can be controlled by combining the following techniques: increasing the number of Monte Carlo samples $m$, reducing the learning rate, and using adaptive optimizers that are robust to noisy gradients. Our experiments show that the purification performance of our algorithm is not affected by the randomness introduced by the Monte Carlo approximation. Nevertheless, we optionally use a random number generator with fixed seeds to generate the noise samples for the Monte Carlo approximation. In all experiments with our method reported in this paper, we activate this fixed-seed option, so that the purification map is deterministic conditional on the input. In the deterministic white-box threat model, the attacker is assumed to know the generator and the seeds (hence no EoT is required).

	\paragraph{Data augmentation for training the classifier} To further enhance adversarial robustness, we augment the training dataset of the classifier with purified samples. Given a clean training sample $x$, we generate its purified version $x_{\mathrm{pur}}$ by applying the purification method described in Algorithm~\ref{alg:pumere} using $x$ as the input. In particular, we deactivate the SAM option during the purification to encourage the classifier to adapt to the original landscape of the score model $s_\theta$. The training dataset is then augmented with the purified samples $\{(x_{\mathrm{pur}}, y)\}$, where $y$ is the label of the clean sample $x$. The classifier is trained on both the original training samples and the purified samples. The loss function for training the classifier is given by
	\begin{equation}
		\mathcal{L} = \mathbb{E}_{(x, y) \sim \mathcal{D}} \left[ \dfrac{1}{2}\ell (h(x), y) + \dfrac{1}{2}\ell (h(x_{\mathrm{pur}}), y) \right],
	\end{equation}
	where $\mathcal{D}$ is the training dataset, $h$ is the classifier, and $\ell$ is the cross-entropy loss function. This data augmentation strategy encourages the classifier to be robust to purified samples.
	
	\section{Experiments}

	We evaluate the adversarial robustness of our purification method against the following attacks:
	\begin{itemize}
		\item gPGD20: a gray-box attack using 20 PGD steps.
		\item BPDA$s$-det: a deterministic white-box adaptive attack using Backward Pass Differentiable Approximation (BPDA) with $s$ steps, where the random number generator and seed are fixed for both the purification and the attack. 
	\end{itemize}
	For the gPGD20 attack, the adversary has access to the classifier but not the purification process, and adversarial examples are generated by attacking the classifier. 
	For the deterministic white-box adaptive attack BPDA$s$-det, the defense is deterministic, and the adversary has full knowledge of both the classifier and the purification process, including the random number generator.

	The BPDA technique is employed in the white-box attacks to estimate the gradient of the loss function w.r.t.\ the input, since the purification map is not readily differentiable due to projection/clipping and iterative optimization.
	In our BPDA implementation, we use the straight-through (identity) backward approximation from Athalye et al.~\cite[Sec. 4.1.1]{pmlr-v80-athalye18a}. In the forward pass, we compute $x_{\mathrm{pur}}=\mathrm{Pur}(x)$, while in the backward pass, we replace $\nabla_x \mathrm{Pur}(x)$ by the identity. Equivalently,
	\[
	\nabla_x \,\ell\!\left(h(\mathrm{Pur}(x)),y\right)\;\approx\;\nabla_x \,\ell\!\left(h(x),y\right)\big|_{x=\mathrm{Pur}(x)}.
	\]
	This is an approximate gradient; accordingly, we use sufficiently many BPDA steps (e.g., $s=20$ and $s=200$) to obtain a strong adaptive attack.
	
	Using these approximated gradients, adversarial examples are generated iteratively. Various optimization-based update rules can be used, such as the Fast Gradient Sign Method (FGSM)~\cite{goodfellowSS14}, PGD~\cite{madry2018towards}, or the Carlini~\&~Wagner (C\&W)~\cite{CW17} attack. In this work, PGD is adopted, as it is substantially more effective than FGSM and achieves performance comparable to other strong iterative optimization-based attacks. 
	
	For purification methods that involve significant randomness, the EoT technique is necessary to correctly estimate gradients. We denote this attack as BPDA$s$+EoT$k$, which runs $k$ statistically independent purification processes to estimate the gradient required for each of the $s$ BPDA steps. This is the standard EoT estimator for $\nabla_x \,\mathbb{E}_{\omega}\!\left[\ell\!\left(h(\mathrm{Pur}_{\omega}(x)),y\right)\right]$, using $k$ Monte Carlo samples at each BPDA step; see Athalye et al.~\cite[Sec. 4.2]{pmlr-v80-athalye18a}.

	\paragraph{Numerical results on CIFAR-10 dataset} We perform the gPGD20 and BPDA20-det attacks for all test samples, while the BPDA200-det and BPDA$200$+EoT$20$ attacks are conducted on a subset of 1,000 test samples due to the limitations of computational resources. In addition, the BPDA200-det attack is evaluated for the $\ell_\infty$, $\ell_2$, and $\ell_1$ perturbations, while the gPGD20 and BPDA20-det attacks are only tested against $\ell_\infty$ adversarial perturbations. The adversarial perturbation budgets for the $\ell_\infty$, $\ell_2$, and $\ell_1$ attacks are set to $8/255$, $1$, and $12$, respectively. The step size of each PGD step is set to $2/255$ for the gPGD20 and BPDA20-det attacks under the $\ell_\infty$ perturbation budget. For the BPDA200-det attacks, the step size is computed as $2.5 \varrho_{\mathrm{adv}}/200$, which yields $0.00039$, $0.0125$, and $0.15$ for the $\ell_\infty$, $\ell_2$, and $\ell_1$ attacks, respectively.

	We employ the trained score-based model from~\cite{NEURIPS2020_92c3b916} without changing the settings. The model is trained with $L=232$ noise levels in a geometric sequence from $\sigma_1 = 50$ to $\sigma_L = 0.01$. We set $m=4$ to approximate the expected reconstruction error. 
	
	The purification budget $\varrho_{\mathrm{pur}}$ is varied in the range $[2,3]$—that is, from $2$ to $3$ times the $\ell_2$ adversarial perturbation budget—to study the trade-off between clean and adversarial accuracies. The sharpness radius must be smaller, yet of the same order of magnitude as, $\varrho_{\mathrm{pur}}$ to ensure that the purified samples lie in flat regions of the expected reconstruction error $R_\theta$. Hence, we set $\varrho_{\mathrm{sam}} = 1.5$. We use the WideResNet-28-10~\cite{zagoruyko2016wide} architecture for the classifier, which is trained on both the original and purified training samples generated with $\varrho_{\mathrm{pur}} = 3$.

	The adversarial accuracies of our method under the gPGD20 attack are summarized in Table~\ref{tab:cifar_results_graybox}. Our method effectively preserves high accuracy on both clean and adversarial data.

	\begin{table}[H]
		\caption{Adversarial accuracies (\%) on CIFAR-10 of our method under the gPGD20 attack with an $\ell_\infty$ perturbation bound of $\varrho_{\mathrm{adv}} = 8/255$.}\label{tab:cifar_results_graybox}
		\centering
		\begin{tabular}{l|c|c}
			\toprule
			& clean & gPGD20\\ 
			\midrule
			Our method & & \\
			$\qquad$ $\varrho_{\mathrm{pur}} = 2$   & \textbf{91.18} & \textbf{87.88}\\  
			$\qquad$ $\varrho_{\mathrm{pur}} = 2.5$ & 88.45          & 86.73\\  
			$\qquad$ $\varrho_{\mathrm{pur}} = 3$   & 85.68          & 84.17\\  
			\bottomrule
		\end{tabular}
	\end{table}

	The adversarial accuracies of our purification method compared to several representative purification approaches under the BPDA$20$-det and BPDA$20$+EoT$10$ attacks are summarized in Table~\ref{tab:cifar_results} and Table~\ref{tab:cifar10_BPDA200+EoT20}, respectively. 
	The proposed method consistently demonstrates superior robustness compared to existing purification techniques. In particular, under the strong BPDA$20$-det attack, it attains an adversarial accuracy of 69.08\%, substantially exceeding the previously reported state-of-the-art accuracy of 39.16\% in~\cite{liu2025towards}. 
	Moreover, Table~\ref{tab:cifar10_BPDA200+EoT20} presents the results obtained under another strong attack configuration (BPDA$200$+EoT$20$). Our method achieves an average adversarial accuracy of 66.23\% across $\ell_\infty$, $\ell_1$, and $\ell_2$ attacks, outperforming other purification methods.

	\begin{table}[H]
		\caption{\textcolor{black}{Adversarial accuracies (\%) on CIFAR-10 for various purification methods under the BPDA20-det attack with an $\ell_\infty$ perturbation bound of $\varrho_{\mathrm{adv}} = 8/255$.}}\label{tab:cifar_results}
		\centering
		\vspace{0.2cm}
		\begin{tabular}{l|c|c}
			\toprule
			& clean & BPDA20-det\\ 
			\midrule
			Our method & &\\
			$\qquad$ $\varrho_{\mathrm{pur}} = 2$   & 91.18  & 48.17\\  
			$\qquad$ $\varrho_{\mathrm{pur}} = 2.5$ & 88.45  & 63.35\\  
			$\qquad$ $\varrho_{\mathrm{pur}} = 3$   & 85.68  & 69.08\\  
			\midrule
			Liu et al.~\cite{liu2025towards} & & \\ 
			\hspace{0.5cm} $\text{DP}_\text{DDPM}$           & 85.94 & 16.8\\
			\hspace{0.5cm} $\text{DP}_\text{DDIM}$           & 88.38 & 4.98\\
			\hspace{0.5cm} $\text{DP}_\text{DDPM}$ with ADDT & 85.64 & 17.09\\
			\hspace{0.5cm} $\text{DP}_\text{DDIM}$ with ADDT & 88.77 & 39.16\\
			\bottomrule
		\end{tabular}
	\end{table}

	\begin{table}[H]
		\caption{\textcolor{black}{Adversarial accuracies (\%) on CIFAR-10 for various purification methods under the BPDA200-det attack and the BPDA$200$+EoT$20$ attack. The latter uses 200 BPDA steps, each using 20 EoT samples. The numerical results for methods marked with $^\star$ are reported from~\cite{li2025adbm}.}}\label{tab:cifar10_BPDA200+EoT20}
		\centering
		\vspace{0.2cm}
		\begin{tabular}{l|c|c|c|c|c}
			\toprule
			& clean & $\ell_\infty$ & $\ell_1$ & $\ell_2$ & Average\\ 
			\midrule
			BPDA200-det & & & & &\\
			$\qquad$ Our ($\varrho_{\mathrm{pur}} = 3)$ & 84.6 & \textbf{67.9} & \textbf{66.9} & \textbf{63.9} & \textbf{66.23} \\   
			\midrule
			BPDA$200$+EoT$20$ & & & & &\\
			$\qquad$ Li et al.~\cite{li2025adbm} $^\star$ & 91.9 & $47.7$ & 53.5 & 63.3 & 53.5 \\
			$\qquad$ Yoon et al.~\cite{pmlr-v139-yoon21a} $^\star$  & 87.93  & 37.65 & 36.87 & 57.81 & 44.11 \\
			$\qquad$ Nie et al.~\cite{pmlr-v162-nie22a} $^\star$ & \textbf{92.5} & 42.2 & 44.3 & 60.8 & 49.1\\
			\bottomrule
		\end{tabular}
	\end{table}
	
	\section{Conclusion}

	In this paper, we have proposed a novel deterministic purification method for enhancing the adversarial robustness of neural network classifiers. Our method is based on the SAM-based optimization of the expected reconstruction error under noise corruption. We have shown that PurSAMERE pushes purified samples toward local maxima of the smoothed density $p_{Y_\sigma}$; in an idealized local regime (Theorem~\ref{thm:local_max_density_informal}) this recovers a local maximizer of the model potential as $\sigma\to 0$. Experimental results on the CIFAR-10 dataset have demonstrated that our method significantly improves adversarial robustness compared to state-of-the-art purification methods, particularly under deterministic white-box attacks when the adversary has full knowledge of the purification process. As a deterministic purification method, our approach does not suffer from decreasing effective robustness under such attacks, which is commonly observed in other stochastic purification methods. Our method thereby provides more reliable adversarial accuracies.
	
	\section{Acknowledgments}
	V.H., S.K., H.R., and R.T. acknowledge funding by the Deutsche Forschungsgemeinschaft (DFG, German Research Foundation) - Project number 442047500 through the Collaborative Research Center ``Sparsity and Singular Structures'' (SFB 1481).
	V.H. acknowledges funding by the German Federal Ministry of Research, Technology and Space (BMFTR) and the Ministry of Economic Affairs, Industry, Climate Action and Energy of the State of North Rhine-Westphalia through the project HC-H2.
	
	\bibliographystyle{plainnat}
	\bibliography{references}
	
	\appendix
	
	\section{Theoretical analysis}

	\subsection{Denoising score matching objective function}~\label{appendix:score_matching_equivalence}
	\begin{lemma}~\label{lem:dsm_objective}
		For all $s_\theta \in L_2(\mathbb{R}^d, p_{Y_\sigma})$, the denoising score matching objective function \eqref{eq:score_matching} satisfies
		\begin{equation}
			\mathbb{E}_{X, \varXi}[\lVert \varXi + \sigma s_{\theta}(X + \sigma \varXi; \sigma) \rVert_2^2] = \sigma^2 \mathbb{E}_{X, \varXi}[\lVert s(X + \sigma \varXi; \sigma) - s_\theta(X + \sigma \varXi; \sigma)\rVert_2^2] + \operatorname{Const},
		\end{equation}
		where the constant term is independent of $s_\theta$.
	\end{lemma}
	
	\begin{proof}
		Following \eqref{eq:mmse}, for all functions $g \in L_2(\mathbb{R}^d, p_{Y_\sigma})$,
		$$
		\mathbb{E}_{X, \varXi}\left[(X - \mathbb{E}[X \mid Y_\sigma = X + \sigma \varXi])^\top g(X + \sigma \varXi)\right] = 0.
		$$
		Therefore, we have
		$$
		\begin{aligned}
			\mathbb{E}_{X, \varXi} &\left[ \lVert \varXi + \sigma s_{\theta}(X + \sigma \varXi; \sigma)\rVert_2^2 \right] 
			=  \mathbb{E}_{X, \varXi}\left[ \lVert \varXi + \sigma s(X + \sigma \varXi; \sigma)\rVert_2^2 \right] \\
			&+ \sigma^2 \mathbb{E}_{X, \varXi}\left[ \lVert s(X + \sigma \varXi; \sigma) - s_\theta(X + \sigma \varXi; \sigma)\rVert_2^2 \right] \\
			&+ 2 \sigma \mathbb{E}_{X, \varXi}\left[ (\varXi + \sigma s(X + \sigma \varXi; \sigma))^\top (s(X + \sigma \varXi; \sigma) - s_\theta(X + \sigma \varXi; \sigma)) \right] 
		\end{aligned}
		$$
		where the first term on the right-hand side is a constant independent of $s_\theta$, and the third term on the right-hand side, which can be expressed as
		$$
		2 \sigma \mathbb{E}_{X, \varXi}\left[(X - \mathbb{E}[X \mid Y_\sigma = X + \sigma \varXi])^\top (s(X + \sigma \varXi; \sigma) - s_\theta(X + \sigma \varXi; \sigma))\right],
		$$
		is equal to zero.
	\end{proof}
	
	\subsection{Proof of Proposition \ref{thm:ere_expansion}}
	\label{appendix:ere_expansion}
	\begin{proof}[Proof of Proposition \ref{thm:ere_expansion}]
		The Taylor expansion of the one-dimensional map $\sigma \mapsto s(x+\sigma\xi,\sigma)$ at $\sigma=0$ is given by
		\[
		s(x + \sigma \xi, \sigma) = s(x, 0) + \sigma \nabla_x s(x, 0) \xi + \sigma  \partial_\sigma s(x,0) + \dfrac{1}{2}\sigma^2 \, \phi(x, \sigma_t, \xi),
		\]
		for some $\sigma_t\in(0,\sigma)$ depending on $x$ and $\xi$, where
		\[
		\phi(x, \sigma, \xi) = \xi^\top \nabla_{x}^2 s(x +\sigma \xi, \sigma)\xi + 2 \partial_\sigma \nabla_{x} s(x +\sigma \xi, \sigma)\xi + \partial_{\sigma \sigma}  s(x +\sigma \xi, \sigma).
		\]
		
		Let 
		\[
		A(x, \xi): = \xi + \sigma s(x, 0) + \sigma^2 \nabla_x s(x, 0) \xi + \sigma^2  \partial_\sigma s(x,0).
		\] 
		Since $\varXi \sim \mathcal{N}(0, \mathbf{I}_d)$, we have
		\[
		\begin{aligned}
			\mathbb{E}_\varXi \!\left[\|A(x, \varXi)\|^2\right] &= \mathbb{E}_\varXi \!\left[\| \varXi + \sigma\,s(x, 0) + \sigma^2 \nabla_x s(x, 0) \varXi + \sigma^2 \partial_\sigma s(x,0) \|_2^2\right] \\
			&= d + \sigma^2 \|s(x,0)\|_2^2 + 2\sigma^2 \operatorname{tr}(\nabla_x s(x,0)) + \mathcal{O}(\sigma^3).
		\end{aligned}
		\]
		
		Moreover, there exist constants $C>0$, and $m\ge 0$, such that for all $x\in K$, all $\sigma\in[0,\sigma_{\max}]$, and all $\xi\in\mathbb{R}^d$,
		\[
		\max_{1\le i\le d} |\phi_i(x,\sigma,\xi)|
		\;\le\; C\bigl(1+\|\xi\|_2^m\bigr),
		\]
		therefore,
		\[
		\mathbb{E}_\varXi \!\left[\|A(x, \varXi) + \dfrac{1}{2} \sigma^3 \phi(x, \sigma_t, \varXi )\|^2_2\right] = \mathbb{E}_\varXi \!\left[\|A(x, \varXi)\|^2\right] + \mathcal{O} (\sigma^3).
		\]
		Finally, we obtain,
		\[
		\mathbb{E}_\varXi \!\left[\| \varXi + \sigma \, s(x + \sigma \varXi, \sigma) \|_2^2 \right] = d + \sigma^2 \|s(x,0)\|_2^2 + 2\sigma^2 \operatorname{tr}(\nabla_x s(x,0)) + \mathcal{O}(\sigma^3).
		\]  
	\end{proof}

	\subsection{Theorem: Local Recovery of Density Maximizer}
	\label{appendix:local_max_density}
	We restate Theorem \ref{thm:local_max_density_informal} in a formal way.

	\begin{theorem}[Local Recovery of Density Maximizer (Formal)]
		\label{thm:local_max_density}
		Let $s_\theta: \R^d \times (0, \sigma_{\text{max}}] \rightarrow \R^d$ be a score model approximating the score function $s(x; \sigma) = \nabla_x \log p_{Y_\sigma}(x)$ for $\sigma \in (0, \sigma_{max}]$. 
		
		\textbf{Assumption 1:} The function $x \mapsto s_\theta(x; \sigma)$ is piecewise affine and globally $L$-Lipschitz on $\R^d$. More precisely, there exists closed sets $\{D_i\}_{i=1}^N$ of $\R^d$ independent of $\sigma$ such that
		\begin{itemize}
			\item (i) $\cup_{i=1}^N D_i = \R^d$,
			\item (ii) $\operatorname{int}(D_i) \cap \operatorname{int}(D_j) = \emptyset$ for $i \neq j$,
			\item (iii) for each $i$, the interior of $D_i$ has positive Lebesgue measure, and the boundary $\partial D_i$ has Lebesgue measure zero.
		\end{itemize}
		For each $i$ and $\sigma \in (0, \sigma_\mathrm{max}]$, $s_\theta(\cdot; \sigma)$ is affine in $\operatorname{int}(D_i)$ and extends continuously to $D_i$. Moreover, there exists a constant $L>0$ such that 
		\[
		\| s_\theta(x; \sigma) -s_\theta(y; \sigma) \|_2 \leq L \| x - y \|_2, \quad \forall x, y \in \R^d, \, \sigma \in (0, \sigma_\mathrm{max}].
		\]
		
		\textbf{Assumption 2:} There exists a function $f_\theta: \R^d \times (0, \sigma_{\text{max}}]  \rightarrow \R$ such that $s_\theta(x; \sigma) = \nabla_x f_\theta(x; \sigma)$. Moreover, for each  $\sigma \in (0, \sigma_{\text{max}}]$, the function $x \mapsto \exp(f_\theta(x; \sigma))$ defines a density function on $\R^d$. That is, $\int_{\R^d} \exp(f_\theta(x; \sigma)) \, \mathrm{d}x = 1$.
		
		\textbf{Assumption 3:} There exists a compact set $D \in \{D_1, \cdots, D_N\}$ and constants $\mu > 0$ and $0 < \varrho < \sigma_\text{max}$, such that $f_\theta(x, \sigma)$ is $\mu$-strongly concave on $D$ for all  $\sigma \in (0, \varrho]$. Moreover, the $\sqrt{d}\,\varrho$-offset of $D$ is non-empty, i.e.,
		\[
		D_{-\sqrt{d}\,\varrho} = \{x \in D: \inf_{y \notin D} \|x - y\|_2^2 \geq d\varrho^2 \} \neq \emptyset,
		\]
		for each $\sigma \le \varrho$, and the (unique) maximizer $x^\star_\sigma$ of $f_\theta(\cdot, \sigma)$ lies in $D_{-\sqrt{d}\,\varrho}$.
		
		Under assumptions 1, 2, and 3, we have that for all $0< \sigma \leq \varrho$, a minimizer $x'_\sigma \in D_{-\sqrt{d}\,\varrho}$ of 
		\[
		R_\theta(x; \sigma):= \E_{\varXi}\left[\| \varXi + \sigma s_\theta(x + \sigma\varXi, \sigma) \|_2^2\right], \quad  \varXi \sim \mathcal{N}(0, \mathbf{I}_d)
		\] satisfies 
		\[
		\| x'_\sigma - x^\star_\sigma \|_2 \leq     \dfrac{\sqrt{C}}{\sigma \mu} \exp \left (-\dfrac{d}{8}\left(\dfrac{\varrho}{\sigma} - 1\right)^2 \right),
		\]
		for some constant $C>0$ independent of $\sigma$.
	\end{theorem}

	Assumption 1 requires two conditions, first the score model is piecewise linear, and second the the partition $\{D_i\}_{i=1}^N$ is independent of $\sigma$. The first condition is satisfied by common neural network architectures using ReLU activation functions and without layer norm. The second condition can be satisfied when the score model is parameterized as $s_\theta(x; \sigma) = \gamma_\theta(x)/\sigma$ where $\gamma_\theta$ is a neural network as proposed in \cite{NEURIPS2020_92c3b916}. Assumption 2 is generally satisfied if the score model is trained well enough since it approximates the score function of a density function. Assumption 3 requires the existence of a strongly concave region of $f_\theta$ where the local maximizer lies strictly inside this region. This assumption is reasonable since the data density function $p_X$ often has multiple local maxima, and the smoothed density $p_{Y_\sigma}$ approximates $p_X$ well for small $\sigma$. The proof of Theorem \ref{thm:local_max_density} relies on the following two lemmas.

	\begin{lemma}
		\label{lem:integral_bound_outside_domain}
		Let $D$ be a closed set in $\R^d$, let $\varXi \sim \mathcal{N}(0, \mathbf{I}_d)$, and let $g: \R^d \rightarrow \R$ be a function in $L^2(p_\varXi)$, i.e., there exists a constant $C > 0$ satisfying 
		\[
		\mathbb{E}_{\varXi}\left[ | g(\varXi) |^2 \right] \leq C^2.
		\]
		For all $\varrho > 0$ such that the $\sqrt{d}\varrho$-offset of $D$ is nonempty, i.e., 
		\[
		D_{-\sqrt{d}\varrho} := \{x \in D: \inf_{y \notin D} \|x - y\|_2 \geq \sqrt{d}\varrho \} \neq \emptyset,
		\] 
		and for any $0 < \sigma < \varrho$,  the following inequality holds:
		\[
		\bigl\lvert \mathbb{E}\left[\mathds{1}_{D^c}(x + \sigma \varXi) g(\varXi) \right] \bigr\rvert \leq C \exp \left (-\dfrac{d}{4}\left( \dfrac{\varrho}{\sigma} - 1 \right)^2 \right), \quad \forall x \in D_{-\sqrt{d}\,\varrho},
		\]
		where $\mathds{1}_{D^c}$ denotes the indicator function of the set $D^c = \R^d \setminus D$.
	\end{lemma}

	\begin{lemma}
		\label{lem:approximate_quadratic_minimizer}
		Let $f: D \rightarrow \R$ where $D$ is a compact set in $\R^d$. Assume that
		$$
		f (x) = (x-x^\star)^\top H(x-x^\star) + g(x), \quad \forall x \in D,
		$$
		for some $x^\star \in D$, where $H$ is symmetric and positive definite with $H \succeq \mu \mathbf{I}_d$ for some $\mu>0$, and where $g: D \rightarrow \R$ is continuous and  satisfies $|g| \leq \varepsilon$ for all $x \in D$, for some $\varepsilon>0$. Then any minimizer $x_{\mathrm{min}}$ of $f$ in $D$ is in the neighborhood of $x^\star$, i.e., $\|x_{\mathrm{min}} - x^\star\|_2 \leq \sqrt{2\varepsilon/\mu}$.
	\end{lemma}

	\subsubsection{Proofs}
	\begin{proof}[Proof of Theorem \ref{thm:local_max_density}]
		Since $s_\theta$ is linear in $D$ where $D$ is a domain satisfying Assumption 3, we have
		$$
		s_\theta(x; \sigma) = H_\theta(\sigma) (x - x^\star_\sigma) + s_\theta (x^\star_\sigma, \sigma), \quad \text{for} \; x \in D
		$$
		where 
		$$H_\theta(\sigma) = \nabla_x s_\theta(x; \sigma), \quad x \in D,$$
		is a constant Hessian matrix of $f_\theta (x, \sigma)$ in $D$. Assumption 3 leads to $H_\theta(\sigma) \preceq - \mu I_d$ and $s_\theta (x^\star_\sigma, \sigma) = 0$. Therefore $s_\theta(x; \sigma) = H_\theta(\sigma) (x - x^\star_\sigma)$ for $x \in D$.
		
		For $x \in D_{-\sqrt{d}\, \varrho}$, the expected reconstruction error~\eqref{eq:ere_approx} can be expressed as
		\[
		\begin{aligned}
			R_\theta(x; \sigma) & = \mathbb{E}_{\varXi} \left[ \lVert \varXi + \sigma s_\theta (x + \sigma \varXi) \rVert_2^2 \right] \\
			& =  \mathbb{E}_{\varXi}  \left[ \lVert \varXi + \sigma H_\theta(\sigma) (x + \sigma \varXi - x^\star_\sigma) \rVert_2^2\right]\\
			& \quad - \mathbb{E}_{\varXi}  \left[ \mathds{1}_{D^c}(x + \sigma\varXi) \lVert \varXi + \sigma H_\theta(\sigma) (x + \sigma\varXi - x^\star_\sigma) \rVert_2^2\right] \\
			& \quad + \mathbb{E}_{\varXi}  \left[ \mathds{1}_{D^c}(x + \sigma\varXi) \lVert \varXi + \sigma s_\theta (x + \sigma \varXi, \sigma) \rVert_2^2\right].
		\end{aligned}
		\]
		Let
		\[
		\begin{aligned}
			g_1(\xi)
			&:= \bigl\|\xi+\sigma H_\theta(\sigma)(x+\sigma\xi-x^\star_\sigma)\bigr\|_2^2, \\
			g_2(\xi)
			&:= \bigl\|\xi+\sigma s_\theta(x+\sigma\xi;\sigma)\bigr\|_2^2.
		\end{aligned}
		\]
		We aim to bound $\mathbb{E}_{\varXi}\left[ | g_2 (\varXi) - g_1(\varXi) |^2 \right]$ by a constant independent of $\sigma$. Since $\| s_\theta(x; \sigma) - s_\theta(y; \sigma) \|_2 \leq L \| x - y \|_2$, $\| H_\theta (\sigma) \|_2 \leq L$ for all $\sigma$. Moreover, there exists a constant $M_D$ independent of $\sigma$ such that $\| x - x^\star_\sigma \| \leq M_D$ as $D$ is a bounded set.
		Therefore, for $\sigma \leq \varrho$ 
		\[
		\begin{aligned}
			g_1(\xi)
			&\le 2\bigl\|(I+\sigma^2H_\theta(\sigma))\xi\bigr\|_2^2
			+2\bigl\|\sigma H_\theta(\sigma)(x-x^\star_\sigma)\bigr\|_2^2 \\
			&\le 2\|I+\sigma^2 H_\theta(\sigma)\|_2^2 \, \|\xi\|_2^2 
			+2\|\sigma H_\theta(\sigma)\|_2^2\|x-x^\star_\sigma\|_2^2 , \\  
			&\le 2(1+L\varrho^2)^2\|\xi\|_2^2
			+2(L\varrho)^2M_D^2 , 
		\end{aligned}
		\]
		and 
		\[
		\begin{aligned}
			g_2(\xi)
			&\le 2\|\xi\|_2^2 + 2\sigma^2\|s_\theta(x+\sigma\xi;\sigma)\|_2^2 \\
			&\le 2\|\xi\|_2^2 + 2\sigma^2\bigl(\| s_\theta(x_\sigma^\star;\sigma) \|_2+L\|x+\sigma\xi - x_\sigma^\star\|_2\bigr)^2 \\
			&\le 2\|\xi\|_2^2 + 2\sigma^2\bigl(L\|x- x_\sigma^\star\|_2 + L\, \sigma \| \xi \|_2 \bigr)^2 \\
			&\le 2(1+2L^2\sigma^4)\|\xi\|_2^2 + 4\sigma^2L^2\|x- x_\sigma^\star\|_2^2 , \\
			&\le 2(1+2L^2\varrho^4)\|\xi\|_2^2 + 4\varrho^2L^2M_D^2. 
		\end{aligned}
		\]
		Since $\varXi \sim \mathcal{N}(0, \mathbf{I}_d)$, the second moments of $g_1(\varXi)$ and $g_2(\varXi)$ are bounded. Therefore, there exists a constant $C > 0$ independent of $\sigma$ such that
		$\mathbb{E}_{\varXi}\left[ | g_2 (\varXi) - g_1(\varXi) |^2 \right] \leq C^2$. We denote $A = \mathbb{E}_{\varXi}\left[ \mathds{1}_{D^c}(x+\sigma \varXi) \left( g_2(\varXi) - g_1(\varXi) \right)\right]$. Using Lemma \ref{lem:integral_bound_outside_domain}  we can bound $A$ as 
		\[
		\begin{aligned}
			| A | \leq  C \exp \left (-\dfrac{d}{4}\left(\dfrac{\varrho}{\sigma} - 1\right)^2 \right).
		\end{aligned}
		\]
		
		Therefore,
		\[
		R_\theta(x; \sigma) \leq \mathbb{E}_{\varXi}  \left[ \lVert \varXi + \sigma H_\theta(\sigma) (x + \sigma \varXi - x^\star_\sigma) \rVert_2^2\right] +   C \exp \left (-\dfrac{d}{4}\left(\dfrac{\varrho}{\sigma} - 1\right)^2 \right).
		\]
		
		The first term in the r.h.s is the main contribution to $R_\theta(x; \sigma)$ and can be expressed as
		\[
		\begin{aligned}
			\mathbb{E}_{\varXi}  \left[ \lVert \varXi + \sigma H_\theta(\sigma) (x + \sigma \varXi - x^\star_\sigma) \rVert_2^2\right]  = &d + \sigma^2 \| H_\theta(\sigma) (x - x^\star_\sigma) \rVert_2^2 \\
			&+ 2 \sigma^2 \operatorname{tr}(H_\theta(\sigma)) + \sigma^4 \operatorname{tr}(H_\theta(\sigma) H_\theta(\sigma)^\top),
		\end{aligned}
		\]
		which is a quadratic function of $x$ with unique minimizer at $x^\star_\sigma$. 
		Since $H_\theta(\sigma) \preceq - \mu I_d$, we have $\sigma^2 H_\theta(\sigma)^\top H_\theta(\sigma) \succeq \sigma^2\mu^2 \mathbf{I}_d$. Using Lemma \ref{lem:approximate_quadratic_minimizer}, we conclude that
		\[
		\| x'_\sigma - x^\star_\sigma \|_2 \leq   \dfrac{\sqrt{2C}}{\sigma \mu} \exp \left (-\dfrac{d}{8}\left(\dfrac{\varrho}{\sigma} - 1\right)^2 \right),
		\]
		which completes the proof.
	\end{proof}
	
	\begin{proof}[Proof of Lemma \ref{lem:integral_bound_outside_domain}]
		By Gaussian concentration of measure for Lipschitz functions, see e.g.\ \cite[Theorem 8.40]{FoucartRauhut2013}, for all $t>0$,  
		\[
		\mathbb{P} (\lVert \varXi \rVert_2 \geq \mathbb{E} \left[ \| \varXi \|_2 \right] + t ) \leq \exp\left(-\dfrac{t^2}{2} \right).
		\]
		Jensen's inequality implies that 
		$
		\mathbb{E}\left [\| \varXi \|_2 \right] \leq \left(\mathbb{E}\left [\| \varXi \|_2^2 \right] \right)^{1/2} = \sqrt{d}$, 
		so that, for all $t>0$,
		\[
		\mathbb{P} (\| \varXi \|_2 \geq \sqrt{d} + t ) \leq \exp\left(-\dfrac{t^2}{2} \right),
		\]
		see also \cite[Equation (8.89)]{FoucartRauhut2013}.
		Setting $\sqrt{d} + t = \sqrt{d} \, \frac{\varrho}{\sigma}$, i.e., $t= \sqrt{d}\left(\frac{\varrho}{\sigma}-1\right)$, this gives
		\[
		\mathbb{P} (\| \varXi \|_2 \geq \sqrt{d} \dfrac{\varrho}{\sigma}) \leq \exp\left(-\dfrac{d}{2} \left(\dfrac{\varrho}{\sigma}-1\right)^2 \right).
		\]
		Since $x \in D_{-\sqrt{d}\,\varrho}$, the open ball $B(x, \sqrt{d}\,\varrho)$ is contained in $D$. Therefore, if $x + \sigma \varXi \notin D$, then necessarily
		\[
		\|x + \sigma \varXi - x\|_2 \ge \sqrt{d}\,\varrho,
		\]
		which implies
		\[
		\|\varXi\|_2 \ge \sqrt{d}\,\frac{\varrho}{\sigma}.
		\]
		Therefore,
		\[
		\mathbb{P} (x + \sigma \varXi \notin D)  \leq \mathbb{P} (\|\varXi\|_2 \geq \sqrt{d}\dfrac{\varrho}{\sigma})  
		\leq \exp\left(-\dfrac{d}{2} \left(\dfrac{\varrho}{\sigma}-1\right)^2 \right).
		\]
		Finally, we have 
		\[
		\begin{aligned}
			\bigl\vert \mathbb{E}_{\varXi}\left[ \mathds{1}_{D^c}(x + \sigma\varXi) g(\varXi)  \right] \bigr\vert 
			& \leq  \mathbb{E}_{\varXi}\left[ \mathds{1}_{D^c}(x + \sigma\varXi) \right]^{1/2} \mathbb{E}_{\varXi}\left[g^2(\varXi)  \right]^{1/2} \\
			& \leq   C  \mathbb{E}_{\varXi}\left[ \mathds{1}_{D^c}(x + \sigma\varXi) \right]^{1/2} 
			=   C \mathbb{P} (x + \sigma \varXi \notin D)^{1/2}   \\
			& \leq  C  \exp\left(-\dfrac{d}{4} \left(\dfrac{\varrho}{\sigma}-1\right)^2 \right),  \\
		\end{aligned}
		\]
		which completes the proof.
	\end{proof}

	\begin{proof}[Proof of Lemma \ref{lem:approximate_quadratic_minimizer}]
		Since $H \succeq \mu \mathbf{I}_d$, we have
		$$
		f(x) \geq \mu \lVert x - x^\star \rVert_2^2 + g(x), \quad \forall x \in D.
		$$
		Let us denote $r = \sqrt{2\varepsilon/\mu}$. For all $x$ such that $\lVert x - x^\star \rVert_2 > r$, we have
		$$
		f(x) > \mu r^2 - \varepsilon = \varepsilon
		$$
		On the other hand, we have $f(x^\star) = g(x^\star) \leq \varepsilon$. Therefore, the minimizer $x_{\text{min}}$ of $f$ in $D$ must satisfy $\|x_{\text{min}} - x^\star\|_2 \leq r = \sqrt{2\varepsilon/\mu}$.
	\end{proof}
	
\end{document}